\documentclass{article}
\usepackage{ijcai19}

\usepackage{times}
\usepackage{soul}
\usepackage{url}
\usepackage[hidelinks]{hyperref}
\usepackage[utf8]{inputenc}
\usepackage[small]{caption}
\usepackage{graphicx}
\usepackage{booktabs}
\usepackage{multicol}
\usepackage{pgfkeys}
\usepackage{subfiles}
\usepackage{booktabs}
\usepackage{microtype} 
\usepackage[cmex10]{amsmath}
\usepackage{amssymb}
\usepackage{fnbreak} 
\usepackage{url}
\usepackage{amssymb}
\usepackage{longtable}
\usepackage{amsmath,amssymb,amsthm}
\usepackage{graphicx} 
\usepackage{graphics, color}
\usepackage{algorithm}
\usepackage{algpseudocode}
\usepackage{longtable}
\usepackage{tabularx}
\usepackage{url}
\newfloat{algorithm}{t}{lop}
\usepackage{enumitem}
\usepackage{thmtools, thm-restate}
\usepackage{xcolor}

\usepackage{qtree} 
\usepackage{subfigure} 
\usepackage{xspace}

\newcommand{\true}{\ensuremath{\top}\xspace}
\newcommand{\false}{\ensuremath{\bot}\xspace}

\newcommand{\floor}[1]{\left\lfloor #1 \right\rfloor}
\newcommand{\ceil}[1]{\left\lceil #1\right\rceil}

\newcommand{\CryptoMiniSAT}{\ensuremath{\mathsf{CryptoMiniSAT}}}

\newcommand{\SAT}{\ensuremath{\mathsf{SAT}}}

\newcommand{\WeightMC}{\ensuremath{\mathsf{WeightMC}}}

\newcommand\blfootnote[1]{%
	\begingroup
	\renewcommand\thefootnote{}\footnote{#1}%
	\addtocounter{footnote}{-1}%
	\endgroup
}

\newcommand{\ApproxMCThree}{\ensuremath{\mathsf{ApproxMC3}}}

\newcommand{\WISH}{\ensuremath{\mathsf{WISH}}}

\algnotext{EndFor}
\algnotext{EndIf}
\algnotext{EndWhile}
\algblockdefx{Do}{EndDo}{\textbf{do}\;}{}
\algnotext{EndDo}
\newcommand{\adder}{{\sf adder}}
\newcommand{\bdd}{{\sf bdd}}
\newcommand{\cardnet}{{\sf cardnet}}

\pgfkeys{
	/cnfxor/.is family, /cnfxor,
	default/.style = {k = {k}, n = {n}, s = {s}, p = {p}, cnfNum = {}, xorNum = {}},
	k/.estore in = \clauseK,
	n/.estore in = \clauseN,
	s/.estore in = \clauseS,
	p/.estore in = \clauseP,
	cnfNum/.estore in = \clauseCNFOverride,
	xorNum/.estore in = \clauseXOROverride,
}

\newcommand\F[1][]{%
	\pgfkeys{/cnfxor, default, #1}%
	F( 
	\ifx\clauseCNFOverride\empty\relax
	\clauseN, \clauseK 
	\else
	\clauseCNFOverride
	\fi
	)
}

\newcommand\Q[1][]{%
	\pgfkeys{/cnfxor, default, #1}%
	Q^{\frac{1}{2}}(\clauseN, 
	\ifx\clauseXOROverride\empty\relax
	\clauseS \clauseN
	\else
	\clauseXOROverride
	\fi
	)
}

\newcommand\FQ[1][]{%
	\pgfkeys{/cnfxor, default, #1}%
	\psi^{\frac{1}{2}}(\clauseN, 
	\clauseK ,
	\ifx\clauseXOROverride\empty\relax
	\clauseS \clauseN
	\else
	\clauseXOROverride
	\fi
	)
}

\renewcommand{\P}[1]{\ensuremath{\mathsf{Pr}\left[#1\right]}}

\newtheorem{theorem}{Theorem}
\newtheorem*{theorem*}{Theorem}

\newtheorem{lemma}[theorem]{Lemma}

\newcommand{\bigO}{\mathcal{O}}

\hypersetup{draft} 

\title{Phase Transition Behavior of Cardinality and XOR Constraints}
\author{
	Yash Pote$^1$\and
	Saurabh Joshi$^2$\And
	Kuldeep S. Meel$^1$\\
	\affiliations
	$^1$National University of Singapore\\
	$^2$IIT Hyderabad, India\\
}
\begin{document}
\maketitle

\begin{abstract}
The runtime performance of modern SAT solvers is deeply connected to the phase transition behavior of CNF formulas. While CNF solving has witnessed significant runtime improvement over the past two decades, the same does not hold for several other classes such as the conjunction of cardinality and XOR constraints, denoted as CARD-XOR formulas. The problem of determining satisfiability of CARD-XOR formulas is a fundamental problem with wide variety of applications ranging from discrete integration in the field of artificial intelligence to maximum likelihood decoding in coding theory. The runtime behavior of random CARD-XOR formulas is unexplored in prior work. In this paper, we present the first rigorous empirical study  to characterize the runtime behavior of 1-CARD-XOR formulas. We show empirical evidence of a surprising phase-transition that follows a non-linear tradeoff between CARD and XOR constraints.  
\end{abstract}

\section{Introduction}

The study of runtime behavior of algorithmic techniques in the context of constraint satisfaction problems (CSP) has been key to several breakthroughs in the design of new solvers~\cite{DM94}. Specifically, a deep
connection was discovered between the density (ratio of the number of clauses to the number of variables) of random
propositional CNF fixed-width (fixed number of literals per clause) formulas and the runtime  behavior of SAT solvers on such formulas. 
For random $k$-CNF formulas, where every clause contains exactly $k$ literals, experiments
suggest a specific phase-transition density, for example 4.26 for random 3-SAT,
but establishing this analytically has been highly challenging~\cite{CP13a}, and it has been
established only for $k=2$~\cite{CR92} and all large enough $k$~\cite{DSS15}. 
A phase-transition phenomenon has also been identified in random XOR formulas (conjunctions of 
XOR constraints). 
Creignou and Daud{\'e}~\shortcite{CD99} proved a phase-transition at density 1 for variable-width 
random XOR formulas.

 Recently, Dudek, Meel, and Vardi~\shortcite{DMV16,DMV17} extended such studies to the conjunction of CNF and XOR constraints, called CNF-XOR formulas. The motivation for their study was the usage of CNF-XOR formulas in the recent hashing-based techniques for the problem of propositional model counting~\cite{Stockmeyer83,CMV13b,CMV16,SM19}.


  Satisfiability of  conjunction of a cardinality constraint and XOR constraints gives rise to an interesting problem, which we shall refer to as 1-CARD-XOR. Given a set of propositional variables, a cardinality constraint (CARD) puts  bounds on how many of these variables can be set to {\tt true}.
It is worth noting that 1-CARD-XOR is still NP-complete~\cite{BMT78}, even to approximate~\cite{ABSS93}, and as our experimental evaluation demonstrates, the study of 1-CARD-XOR alone is computationally expensive. Furthermore, 1-CARD-XOR formulas are necessary and sufficient for maximum likelihood decoding (MLD), one of the most crucial problems in coding theory, in which one seeks to extract the maximum amount of information from a noisy channel. The problem of maximum likelihood decoding is equivalent to determining satisfiability of a 1-CARD-XOR formula. Consequently, MLD has been subject to theoretical and practical investigations for over 50 years ~\cite{Chase85,TV15}.


Generalization of a cardinality constraint is a Pseudo-Boolean (PB) constraint which enforces bounds on the summation of the weights of the propositional variables that can be set to {\tt true}.
A variant and a more generalized version of 1-CARD-XOR problem is the satisfiability of conjunction of CNF constraints, one Pseudo-Boolean constraint and random XOR constraints, denoted as CNF-PB-XOR 
formulas. Formulas of this kind play a crucial role in solving one of the fundamental problems in artificial intelligence: discrete integration.
  Given a set of constraints as a Boolean formula $F$ and a weight function $W$, the problem of discrete integration is to compute the total weight of the set of solutions of input constraints. This has applications in 
numerous areas, including probabilistic reasoning, machine learning, 
planning, statistical physics, inexact computing, and 
constrained-random verification
\cite{JS96,MP99,Bacchus2003,Sang04combiningcomponent,DH07,GSS08,Mur12,EGSS14a}.

Recently, two hashing-based approaches have been proposed for discrete integration: {\WISH}~\cite{EGSS13c} and {\WeightMC}~\cite{CFMSV14}.  Both of these approaches provide strong Probably Approximately Correct (PAC)-style guarantees, i.e., $(\varepsilon,\delta)$ guarantees. The core idea of {\WeightMC} is to partition the problem of discrete integration into linearly many {\em regions} such that the weight of satisfying assignments in each of the {\em regions} is {\em almost equal}; thereby allowing the usage of hashing-based unweighed counting techniques for each of the regions. Each of the regions can be represented by the conjunction of $F$ and one Pseudo-Boolean (PB) constraint. Consequently, the underlying SAT solver invoked during unweighted counting subroutine needs to handle the CNF-PB-XOR formulas.
 While the elegant formulation of {\WISH} and {\WeightMC} promises scalability and strong theoretical guarantees, {\WISH} and {\WeightMC} have not witnessed scalability similar to that of unweighted counting algorithms such as {\ApproxMCThree}. Unlike {\CryptoMiniSAT}, which is optimized for CNF-XOR formulas, to the best of our knowledge, there do not exist specialized solvers that can handle CNF-PB-XOR formulas efficiently. Design of solvers to efficiently  handle CARD-XOR formulas alone would push the boundaries of state-of-the-art techniques to handle several problems of interest~\cite{Par18}.

The phase-transition behavior of CNF constraints has been analyzed to explain runtime behavior of {\SAT} solvers~\cite{AchCoj08}. Furthermore, the study of Dudek et al~\shortcite{DMV16,DMV17} contributed to the development of a new architecture for handling CNF-XOR constraints~\cite{SM19}. We believe that analysis of the phase-transition phenomenon for CARD-XOR formulas would be pivotal  towards demystifying the runtime behavior of the current state of the art solvers.  
Therefore, a deeper understanding of the runtime behavior of CSP/SAT solvers for CARD-XOR constraints can have far-reaching consequences.

 The primary contribution of this work is the first rigorous empirical study to characterize the runtime behavior of 1-CARD-XOR formulas. In particular:
 \begin{enumerate}
 	\item We prove (in Section~\ref{sec:analysis}) upper and lower bounds on the location of the 1-CARD-XOR  phase-transition region.
 	
 	\item 	We present (in Section~\ref{sec:experiments}) experimental evidence for phase transition behavior of 1-CARD-XOR formulas, henceforth known as 1-CARD-XOR phase-transition that follows a non-linear trade-off between $k$-CNF clauses and XOR clauses.
 	
	\item  We demonstrate that the runtime behavior of SAT solver around phase transition is reminiscent of random CNF formulas but is surprisingly different from CNF-XOR formulas. This observation underscores the need for further exploration in this direction for deeper understanding.

 \end{enumerate}
 
 The surprising nature of our observations opens up future directions of research and we hope that a deeper understanding would lead to the design of efficient CARD-XOR solvers in the future. The rest of the paper is organized as follows. We discuss notations and preliminaries in Section~\ref{sec:prelims}. We survey related work in Section~\ref{sec:related}. We present a theoretical analysis to obtain lower and upper bounds on the location of phase transition in Section~\ref{sec:analysis}. We then present the empirical behavior of 1-CARD-XOR constraints in Section~\ref{sec:experiments}.  We finally conclude in Section~\ref{sec:conclusion}.

\section{Notations and Preliminaries} \label{sec:prelims}

Let $X = \{x_1,\cdots,x_n\} $ be a set of propositional variables and let $F$ be a formula defined over $X$. A \textit{satisfying assignment} or a \textit{witness} of $F$ is an assignment of truth values to the variables in $X$ such that $F$ evaluates to true. Let $\#F$ denote the number of satisfying assignments of F. We say that $F$ is satisfiable (or SAT) if $\#F>0$ and unsatisfiable (or UNSAT) if $\#F=0$.
 
A single XOR constraint (also called a {\em XOR clause}) over $X$ is specified as $ a_{1}x_{1} \oplus a_{2}x_{2} \oplus \cdots \oplus a_{n}x_{n} = b_{0}$, where all $a_{i},b_{j} \in \{0,1\}$. 
Satisfiability of a system of $m$ XOR constraints (XORSAT) over $n$ variables can be thought of as a matrix $A \in \{0,1\}^{m \times n} $, a vector $b \in \{0,1\}^m$, and a variable vector $x \in \{0,1\}^n$ which satisfy $Ax=b$. The density $s$ of a system of $m$ XOR constraints is the ratio $s=m/n$.

To generate a random XORSAT instance, we create such a matrix $A$ and the vector $b$ with each element either $0$ or $1$ with probability $\frac{1}{2}$.    
Let the random variable $\Q$ denote such a randomly generate XORSAT instance over $n$ with $\ceil{sn}$ XOR clauses, where $s=\frac{m}{n}$ is called the XOR density. On expectation
a XOR clause in such an instance would have $\frac{n}{2}$ variables.

An at-most-$k$ cardinality constraint is satisfiable by an assignment if and only if at most $k$ of the $n$ literals are set to \true by that assignment. The set of satisfying assignments for this constraint forms a Hamming ball of radius $k$, which has volume $\sum_{w=0}^{k} {{n} \choose {w}}$.
Let $\F$  represent the CNF encoding of the at-most-k cardinality constraint over $n$ variables. We will use $\#\F$ to represent the number of solutions to $\F$, $\#\F = \sum_{w=0}^{k} {{n} \choose {w}}$.  

A 1-CARD-XOR formula is the conjunction of some number of XOR clauses and a cardinality constraint. For fixed positive integers $k$ and $n$, and a fixed positive real number $s$, let the random variable $\FQ$ denote the formula $\Q \wedge \F$.

We use $\P{E}$ to denote the probability of an event $E$. 
We say that an infinite sequence of random events $E_1, E_2, \cdots $ occurs 
\emph{with high probability} (denoted, w.h.p.) if $\lim\limits_{n \to \infty} \P{E_n} = 1$.
Let $H(\mu)= - \mu\log_{2}(\mu)-(1-\mu)\log_{2}(1-\mu)$ denote the binary entropy function.

\section{Related Work} \label{sec:related}
\subsubsection{Phase Transitions}
Motivated from statistical physics, the study of satisfiability of random constraint satisfaction problems led to the observation of a phase transition behavior~\cite{KS94}. In particular, the probability of satisfiability was observed to undergo a sharp transition from one to zero at the {\em critical} density, defined as the ratio of number of clauses to number of variables. For example, for random 3-SAT, the critical point was observed at density 4.26. Theoretical investigations into the location of random k-SAT have led to the precise identification of density for $k=2$ and existence of a sharp transition for large $k$~\cite{DSS15}. 

Furthermore, a phase-transition phenomenon has also been identified in
random $l$-XOR formulas (conjunctions of XOR constraints of length $l$), for $l \geq 1$ , without specifying an exact location for the phase-transition \cite{CDH03}. Pittel and Sorkin~\shortcite{PS16} identified the location of the phase transition for $l$-XOR formulas for $l > 3$. Dudek \textit{et al.}~\shortcite{DMV16} first studied the satisfiability threshold for the conjunction of random $k$-CNF and random variable-width XOR formulas. As is the case with random $k$-CNFs, experiments confirm that the hardest instances are at the critical threshold for $k$-CNF-XOR formulas~\cite{DMV17}. To the best of our knowledge, no prior work exists regarding the study of phase transition for a formula with one cardinality constraint in conjunction with a set of random variable-width XOR constraints (1-CARD-XOR formulas). 

\subsubsection{Cardinality Encodings}
Cardinality (CARD) constraints  naturally arise in many different contexts, such as computer tomography~\cite{GGP99}, MaxSAT algorithms~\cite{FM06}, radio frequency assignment~\cite{YD13}, product configuration~\cite{YD13}, program repair~\cite{joshikroening-fm15} and weighted counting problems~\cite{Par18}. Due to their ubiquity in several application domains, several encodings have been developed which translate them into the Boolean CNF form such as the Totalizer encoding~\cite{totalizer}, the Sequential counter~\cite{Sinz05}, Adder~\cite{ES06}, BDD based encoding~\cite{bdd},  Cardinality Networks~\cite{CardinalityNA09}, and the like. These encodings exhibit different characteristics in terms of their size (e.g., number of clauses and number of variables) and whether they preserve arc-consistency, i.e., the solver is able to detect inconsistencies by unit propagation alone~\cite{ZY00}. Therefore, we focus on observing the repeatability of behavior across different encodings before drawing a conclusion in our study.

\section{Establishing a Phase-Transition}
\label{sec:analysis}
\renewcommand{\theenumi}{(\alph{enumi})}
In this section we will first define and show the existence of a phase transition phenomenon in 1-CARD-XOR formulas. A phase transition boundary is defined by a function $\phi$ such that, w.h.p., a random formula with $s<\phi$
is satisfiable, and becomes unsatisfiable as soon as $s>\phi$.
Note that the random variable denoting a 1-CARD-XOR formula is $\FQ = \F \wedge \Q$.

\begin{theorem}
	\label{thm:existence}
	 (\cite{DMV16}, Theorem 1)

	If $\phi{(k/n)} = \frac{1}{n} \log_2{(\#\F)}$, then for all $k\geq1$ and $s\geq0$:
	\begin{enumerate}
	 \item If $s < \phi{(k/n)} $, then w.h.p. $\FQ$ is satisfiable.
	 \item If $s > \phi{(k/n)} $, then w.h.p. $\FQ$ is unsatisfiable.
	\end{enumerate}
\end{theorem}
	
This is a special case of Theorem 1 given in \cite{DMV16}, which establishes the existence of a phase transition for random CNF-XOR formulas.
 The region of satisfiability is sharply separated from the region of unsatisfiability by the function $\phi{(k/n)}$.  
Since we give an explicit function $\phi$ in Lemma \ref{lemma:lowerbound} and \ref{lemma:upperbound}, we are able to show sharp numerical bounds on the phase transition boundary. We plot the transition function, $\phi{(k/n)}$, with a red line in all figures in this paper.

Next, we use a result from \cite{DMV16}, which gives us the probability of a CNF being satisfiable when conjuncted with some number of XOR constraints, in terms of the count of solutions of that CNF. Using Theorem \ref{thm:satisfiability:} we relate the satisfiability threshold with the model count of any formula when conjuncted with \textit{random} XORs.

\begin{theorem}
	\label{thm:satisfiability:}
     (\cite{DMV16} , Lemma 7 and 12)
     
      Let $\alpha\geq 1$, $s \geq 0$, $n \geq 0$, and let F be a formula defined over $\{ X_1,\ldots,X_n \}$. Then 
	\begin{enumerate}
		
	 \item $\P{ F \wedge \Q  \text{ is satisfiable}   \bigm\vert \#F \geq 2^{\ceil{sn} +\alpha} }$ $ \geq 1- 2^{-\alpha}$.
	 \item $\P{F \wedge \Q \text{ is unsatisfiable} \bigm\vert \#F \leq 2^{\ceil{sn} -\alpha} } \geq 1- 2^{-\alpha}$.
	\end{enumerate}
\end{theorem}

Since the bounds given in Theorem \ref{thm:existence} are not in closed form, we provide analytic bounds which are weaker. The separation in the lower and upper bound is exactly 1 for $ k> \floor{n/2}$ while for  $k \leq \floor{n/2}$ the separation is $\bigO(\log(n)/n)$

\begin{theorem}
	\label{thm:bounds}$\FQ$ is satisfiable w.h.p. if:
	\begin{enumerate}
		\item $s <  H(k/n) - \log(8k(1-k/n))/n$, and $0<k<n/2$
		\item $s < 1 - 1/n $, and $n/2\leq k\leq n$  
	\end{enumerate}
	$\FQ$ is unsatisfiable  w.h.p. if:
	\begin{enumerate}
		\setcounter{enumi}{2}
		\item If $s >  H(k/n) $, and $0<k<n/2$
		\item If $s > 1  $,  and $n/2\leq k\leq n$ 
	\end{enumerate}
\end{theorem}
\begin{proof}
Part(a) and (b) follow from Lemma \ref{lemma:lowerbound:approx} and Part(c) and (d) follow from Lemma \ref{lemma:upperbound:approx} presented in Sections \ref{sec:analysis}.1 and \ref{sec:analysis}.2 respectively.  
\newline
\end{proof}
We will use the facts that $\#\F = \sum_{w=0}^{k} {{n} \choose {w}}$ and $\FQ = \F \wedge \Q$.

	
We use a commonly known bound for summation of binomial coefficients, 
\begin{lemma}(\cite{MS78}, Lemma 10.8). \newline
$  2^{nH(k/n)}/ (8k(1-k/n)) \leq  \sum_{w=0}^{k} {{n} \choose {w}} \leq 2^{nH(k/n)}$, \newline for $0< k\leq n/2$ and for all $n \geq 1$.
\end{lemma}

\subsection{Lower bound}
\begin{lemma}
	\label{lemma:lowerbound}	
	Let $k \geq 2$ and $ s\geq 0.$ If $s < \frac{1}{n} {log_{2} \#\F }$ as $\lim\limits{n \to \infty}$, then w.h.p. $ \FQ$ is satisfiable.
\end{lemma}
\begin{proof}	Since $2^{s}  < \#\F^{1/n}$ we can choose $\delta > 0$ and $ N > 0$ such that $ 2^{s + \delta + 1/N} < \#\F^{1/n}$. We can always find a small enough $\delta$ and a sufficiently large $N$ such that this is true. Since we are concerned with behavior asymptotic in $n$, we consider only $n>2N$. Then we have $2^{sn + \delta n + 2} < \#\F$ and so $2^{\ceil{sn} + \delta n + 1} < \#\F$. Let $\alpha = \delta n + 1$, so we get $2^{\ceil{sn}+\alpha} \leq  \#\F.$  Using Theorem \ref{thm:satisfiability:}a we see that $\P{\FQ\text{ is SAT} \bigm\vert \#\F\geq2^{\ceil{sn} + \alpha} } \geq 1- 2^{-\alpha}$. Since $\lim \limits_{n \to \infty} 1-2^{-\delta n - 1}$ converges to 1, $\FQ$ is satisfiable w.h.p.
\end{proof}

\begin{lemma}
	\label{lemma:lowerbound:approx}
	For $s \geq 0 $ and $ 0 < k < n$ and $\lim\limits{n \to \infty}$, $\FQ$ is satisfiable w.h.p. if:
	\begin{enumerate}
	\item $k\leq n/2 $ and $ s <  H(k/n) - \log(8k(1-k/n))/n $
	\item $k> n/2$ and $ s<1-1/n $
	\end{enumerate}
\end{lemma}
\begin{proof}
  For $n/2 < k \leq n$ , $\#\F > 2^{n-1}$. Observing that $s<1-1/n<\frac{1}{n} \log_{2}\#\F$ we see that Part(b) is an immediate consequence of Lemma \ref{lemma:lowerbound}.	
 
	Using the bound shown in Lemma , for $0< k<n/2$ we get that		
. Since $2^{s}  < 2^{H(k/n) - \log(8k(1-k/n))/n}$ we can choose $\delta > 0$ 	and $ N > 0$ such that $ 2^{s + \delta + 1/N} < 2^{H(k/n) - \log(8k(1-k/n))/n}$. We can always find a small enough $\delta$ and a sufficiently large $N$ such that this is true. Since we are concerned with behavior asymptotic in $n$, we consider only $n>2N$. Then we have $2^{sn + \delta n + 1} < 2^{nH(k/n) - \log(8k(1-k/n))}$ and so $2^{\ceil{sn} + \delta n + 1} < 2^{nH(k/n) - \log(8k(1-k/n))}$. Let $\alpha = \delta n + 1$, so we get $2^{\ceil{sn}+\alpha} < 2^{nH(k/n) - \log(8k(1-k/n))}.$  Using Theorem \ref{thm:satisfiability:}a we see that $\P{\FQ\text{ is SAT} \bigm\vert 2^{nH(k/n) - \log(8k(1-k/n))}\geq \right.$ $\left.2^{\ceil{sn} + \alpha} } \geq 1- 2^{-\alpha}$. Since $\lim \limits_{n \to \infty} 1-2^{-\delta n - 1}$ converges to 1, $\FQ$ is satisfiable w.h.p.
\end{proof}

\subsection{Upper bound}
\begin{lemma}
	\label{lemma:upperbound}	
	Let $k \geq 2, s\geq 0,$ and $r\geq 0$. If $s > \lim\limits_{n \to \infty} \frac{1}{n}{log_{2}\#\F}$, then w.h.p. $ \FQ$ is unsatisfiable.
\end{lemma}
\begin{proof}	Since $2^{s}  > \#\F^{1/n}$ we can choose $\delta > 0$ 	and $ N > 0$ such that $ 2^{s - \delta - 1/N} > \#\F^{1/n}$. We can always find a small enough $\delta$ and a sufficiently large $N$ such that this is true. Since we are concerned with behavior asymptotic in $n$, we consider only $n>N$. Then we have $2^{sn - \delta n -1} > \#\F$ and so $2^{\ceil{sn} - \delta n - 1} > \#\F$. Let $\alpha = \delta n + 1$, so we get $2^{\ceil{sn}-\alpha} > \#\F.$  Using Theorem \ref{thm:satisfiability:}b we see $\P{\FQ\text{ is UNSAT} \bigm\vert \#\F\geq2^{\ceil{sn} - \alpha} } \geq 1- 2^{-\alpha}$. Since $\lim \limits_{n \to \infty} 1-2^{-\delta n - 1}$ converges to 1, $\FQ$ is satisfiable w.h.p.
\end{proof}
\begin{lemma}
	\label{lemma:upperbound:approx}
	For $s \geq 0 $ , $ 0 \leq k\leq n$ and $\lim\limits{n \to \infty}$, $\FQ$ is unsatisfiable w.h.p. if:
	\begin{enumerate}
	\item $k<n/2 $ and $ s \geq  H(k/n) $
	\item $k\geq n/2$ and $ s > 1 $
	\end{enumerate}
\end{lemma}
\begin{proof}
 For $n/2 \leq k < n$ observing that $s>1>\frac{1}{n} \log_{2}\#\F$ we see that Part(b) is an immediate consequence of Lemma \ref{lemma:upperbound}.

	 Since $2^{s}  > 2^{H(k/n)}$ we can choose $\delta > 0$ 	and $ N > 0$ such that $ 2^{s - \delta - 1/N} > 2^{H(k/n)}$. We can always find a small enough $\delta$ and a sufficiently large $N$ such that this is true. Since we are concerned with behavior asymptotic in $n$, we consider only $n>N$. Then we have $2^{sn - \delta n -1} > 2^{nH(k/n)}$ and so $2^{\ceil{sn} - \delta n - 1} > 2^{nH(k/n)}$. Let $\alpha = \delta n + 1$, so we get $2^{\ceil{sn}-\alpha} > 2^{nH(k/n)}.$  Using Theorem \ref{thm:satisfiability:}b we see that $\P{\FQ\text{ is UNSAT} \bigm\vert 2^{nH(k/n)}\geq2^{\ceil{sn} - \alpha} } \geq 1- 2^{-\alpha}$. Since $\lim \limits_{n \to \infty} 1-2^{-\delta n - 1}$ converges to 1, $\FQ$ is satisfiable w.h.p.
\end{proof}

\begin{figure}
	\begin{center}
		\includegraphics[width= 1\columnwidth]{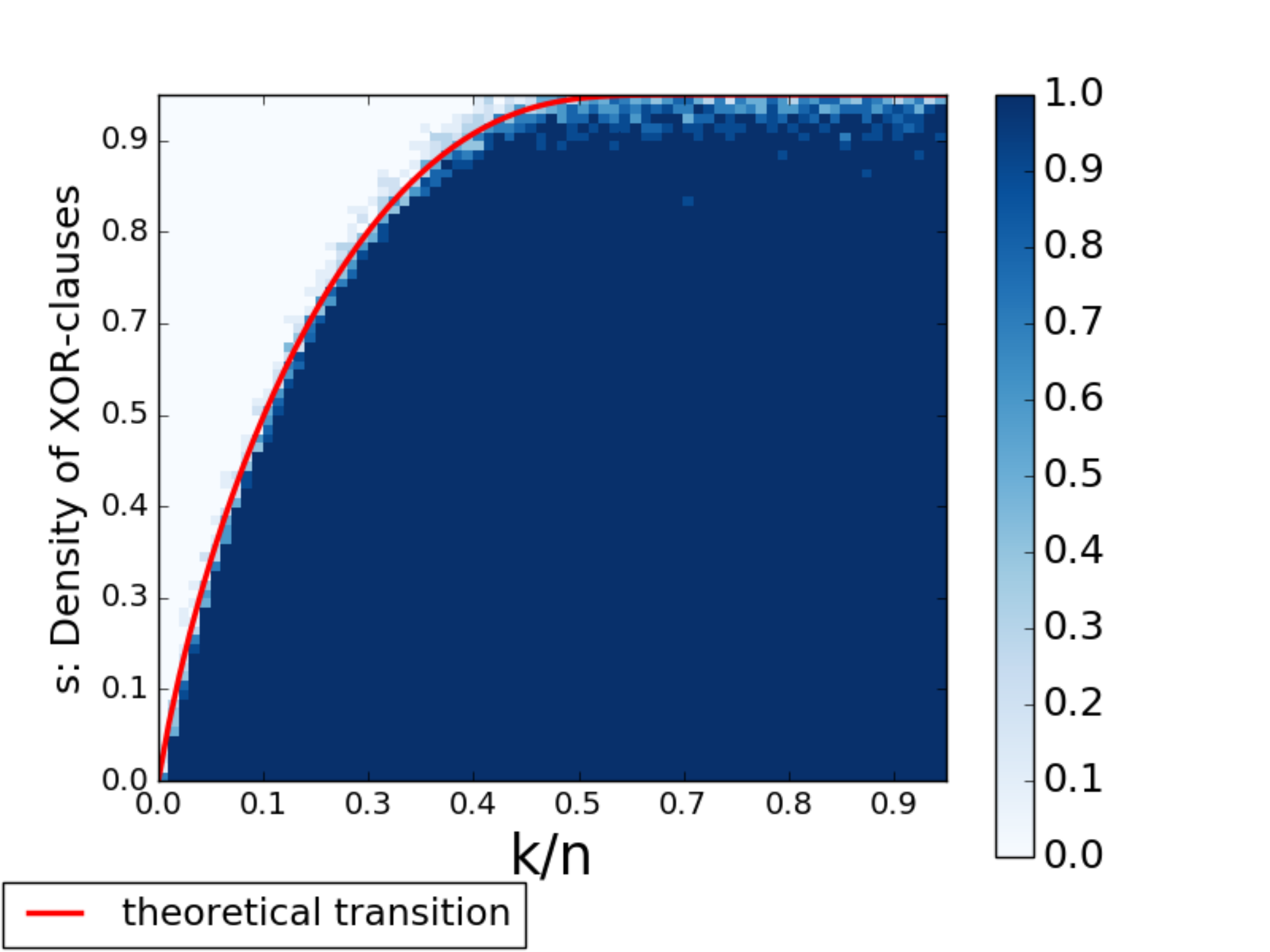}
	\end{center}
	\caption{This plot shows the satisfiability behavior for $n=75$. The darker shade of blue indicates the instances which were satisfiable with higher probability and the red line shows the theoretically derived phase transition. (Best viewed in color)} 
	\label{fig:satisfiablitiy}
\end{figure}
\begin{figure*}
	\begin{multicols}{3}
		\includegraphics[width=1\linewidth]{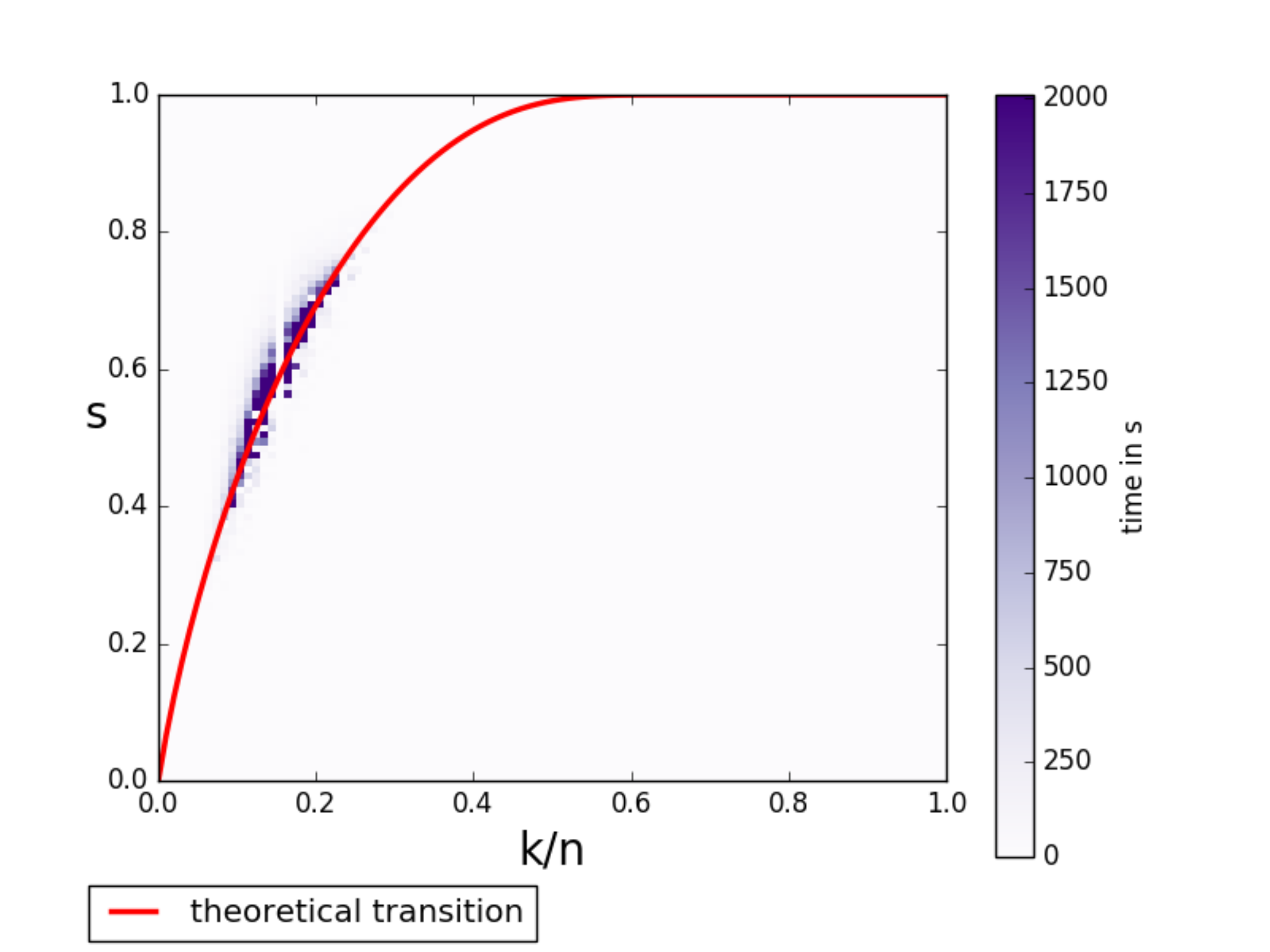}\par
		\captionof{figure}*{$n=100$}
		\includegraphics[width=1\linewidth]{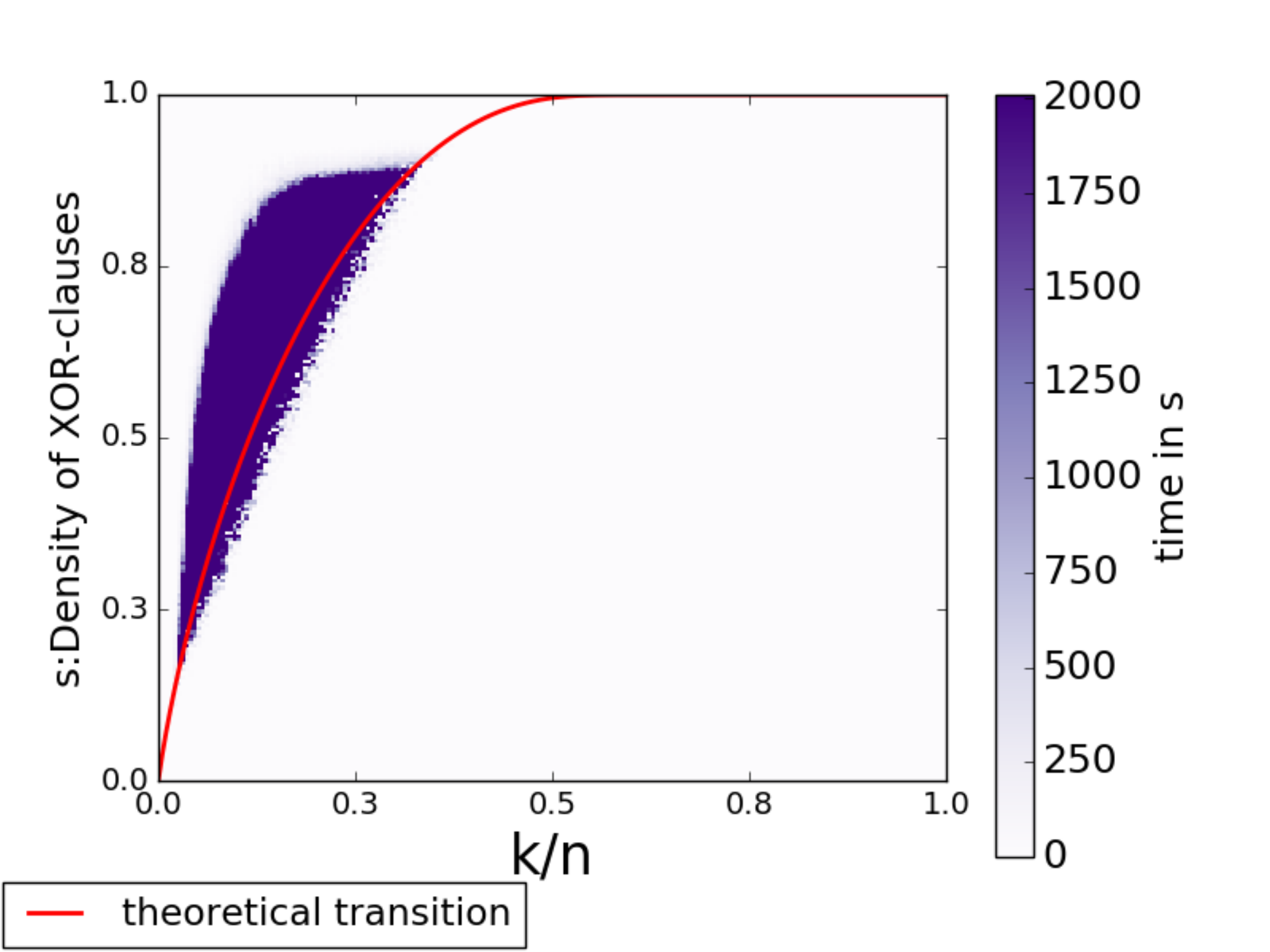}\par 
		\captionof{figure}*{$n=200$}
		\includegraphics[width=1\linewidth]{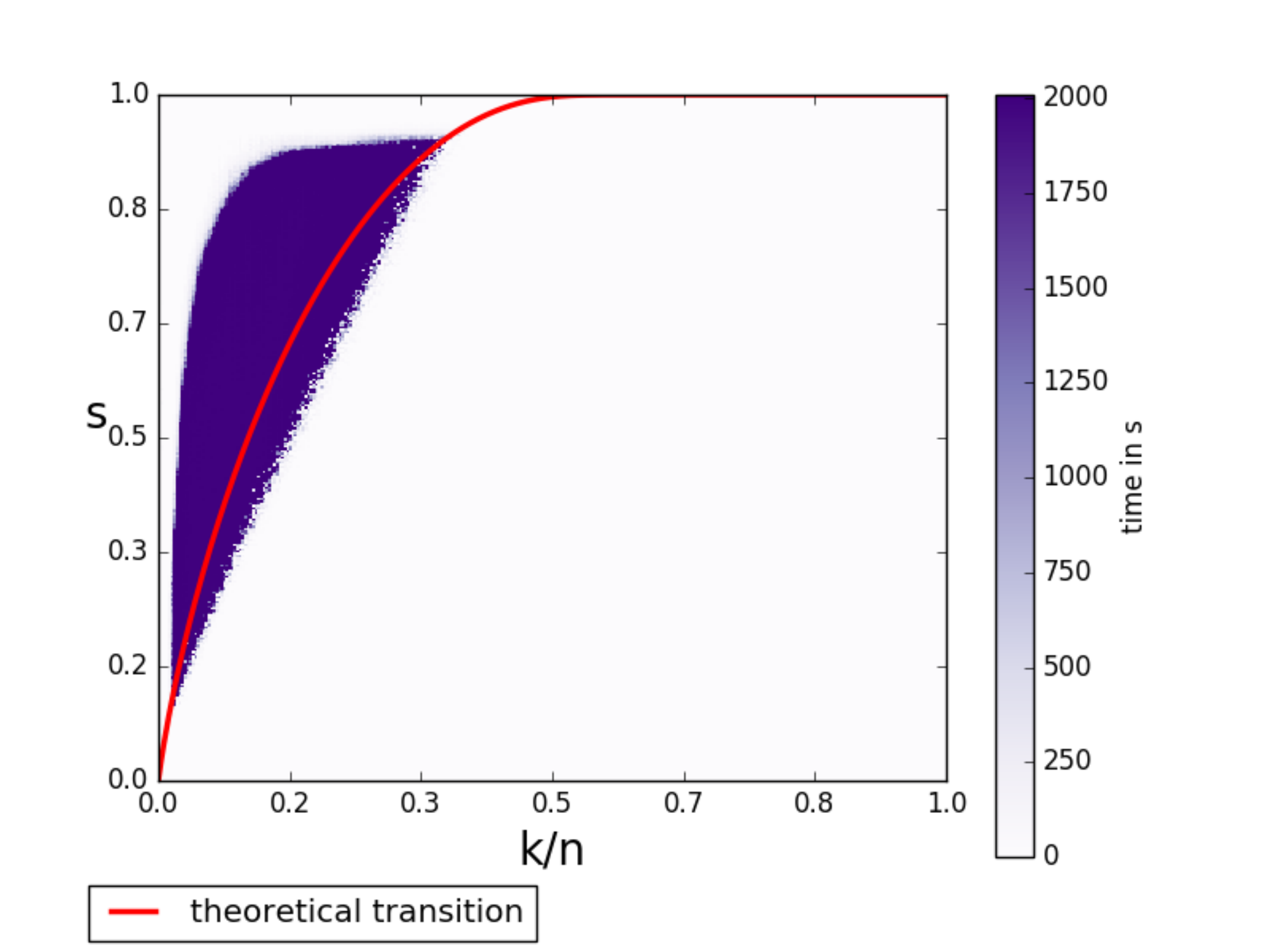}\par
		\captionof{figure}*{$n=300$}
	\end{multicols}
	\caption{Three plots for $\cardnet$ encoding. Each plot shows the runtime behavior for a different number of variables in the following order: (a) $n=100$ , (b) $n=200$, and (c) $n=300$. The purple region is where the blowup in runtime was observed. The red line indicates the phase transition. (Best viewed in color) \label{fig:ncomp}}
\end{figure*}
\section{Experimental Results} \label{sec:experiments}
\subsubsection{Experimental Setup}
\blfootnote{The data and scripts associated with this project are available at \url{https://github.com/meelgroup/1-CARD-XOR} } \label{sec:experiments:scaling:setup}
For every value of $n$, $k$ and $s$ we generate $\FQ$ in the following manner. To uniformly choose an XOR-clause, we include each variable of $\{x_1,\cdots,x_n\}$ with probability $\frac{1}{2}$ in the XOR clause. We then choose exactly one of $\{0,1\}$ with probability $\frac{1}{2}$ in the XOR clause. We repeat such uniform sampling of XOR clauses $\ceil{sn}$ many times and conjunct all such XOR-clauses to build $\Q$. Finally, we conjuct $\Q$ with $\F$ to generate $\FQ$. 

We performed experiments with $9$ values of $n$. For each $n\in \{ 50,75,100,125,150,175,200,250,300\}$ we generated an instance of 1-CARD-XOR for all values of $k \in [0,n]$ and $\ceil{sn} \in [1,n]$. We repeated the experiments 10 times for each data point with $n=\{75,100\}$ to get a finer estimate on the satisfiability at the transition threshold. We were not able to run experiments for values of $n$ significantly larger than those listed above due to computational constraints.

Since \CryptoMiniSAT~\cite{SNC09} is a specialized solver which handles XOR clauses in combination with CNF clauses quite efficiently, we use it as the SAT solver in our experiments. We used a high performance computing cluster of 25 nodes for our experiments. Each node consisted of an Intel\textsuperscript{\textregistered} Xeon\textsuperscript{\textregistered} E5-2690 v3 CPU with 24 cores and 96GB of RAM divided evenly among the cores, with each core having access to 4GB of RAM. Each experiment was conducted on a single core. Our experimental evaluation used over 80,000 CPU hours. 

It is known that encoding cardinality constraint using different encodings may result not only in different sizes for $\F$ but also may have an impact on the performance of the solver.
To investigate the impact of various encodings of cardinality constraints for 1-CARD-XOR formulas, for each value of $n$, $s$ and $k$ we experimented with three cardinality encodings.
  \begin{itemize}[itemsep=1pt]
	\item Adder: $\bigO(n)$ clauses, no arc consistency ~\cite{ES06}
	\item BDD:   $\bigO(n\cdot k)$ clauses, preserves arc-consistency, is equivalent to $LT^{n,k}_{SEQ}$ encoding ~\cite{Sinz05} 
	\item Cardinality Network:  $\bigO(n\cdot log^2 k)$ clauses, preserves arc consistency ~\cite{CardinalityNA09} 
\end{itemize} 
In our experiments, we have used the PBLib ~\cite{pblib} tool to encode our constraints.
A timeout of $2000$ seconds was used for all experiments. 
\subsubsection{Polarity Caching}
The SAT solver goes through an iterative process of search and inference. During search, it selects an unassigned variable and then decides on the truth value (polarity) to be given to this variable.
During inference phase it explores the implications of these choices until we either get a contradiction or a satisfying assignment or nothing further can be inferred and the solver has to  make another variable selection. Polarity selection heuristics decide which truth value, either {\tt true} or {\tt false} to assign to the selected variable. When the solver backtracks due to a contradiction, it must explore  the other choice for the truth value. A heuristic which works very well in practice is polarity caching \cite{DBLP:conf/sat/PipatsrisawatD07}, which involves remembering the previous successful choice made on a particular variable. Another simpler heuristic is setting the polarity to {\tt false}, which instructs the solver to always explore the {\tt false} branch before the {\tt true}. As discussed later, setting the polarity to {\tt false} always significantly improves the runtime performance of the solver. Therefore, the experiments concerning runtime performance of SAT solver with respect to other parameters were performed with setting polarity flag to {\tt false} in \CryptoMiniSAT.


\paragraph{Results}
The objective of our experimental evaluation is to answer these four research questions:
\begin{enumerate}[itemsep=1pt,leftmargin=9.5mm]
	\item[\bf{RQ1.}] How does the satisfiability of $\FQ$ vary, as parameters $k$  and XOR clause density $s$ vary ?
	\item[\bf{RQ2.}] How does the runtime performance of SAT solver for $\FQ$ vary with respect to $n$, $k$, $s$?
	\item[\bf{RQ3.}] How do the different encodings affect the runtime performance of the SAT solver for $\FQ$?
	\item[\bf{RQ4.}] How do the different branching heuristics affect the runtime performance of a SAT solver on $\FQ$?
\end{enumerate}
We now first present detailed analysis below and then summarize our main conclusions. 

\paragraph{\bf{RQ1.}}Figure~\ref{fig:satisfiablitiy} shows the satisfiability of random instances of $\FQ$ as the density of XOR-clauses $s$  and the upper bound on at-most-k constraint $k$ varies. The y-axis indicates the density $s$ and the x-axis indicates $k/n$. Each point on the plot is color-coded to represent satisfiability of the corresponding $\FQ$. The dark blue color indicates that the corresponding formula is satisfiable while light color indicates unsatisfiability. The red curve represents the theoretical phase transition curve,  obtained in Section \ref{sec:analysis}, i.e., a point in the region under the curve is likely to be satisfiable while a point in the region over the curve is likely to be unsatisfiable. We observe that the empirically observed behavior agrees with the analysis, thus demonstrating the tightness of our analysis. 

\begin{figure*}
	\begin{multicols}{3}
		\includegraphics[width=1\linewidth]{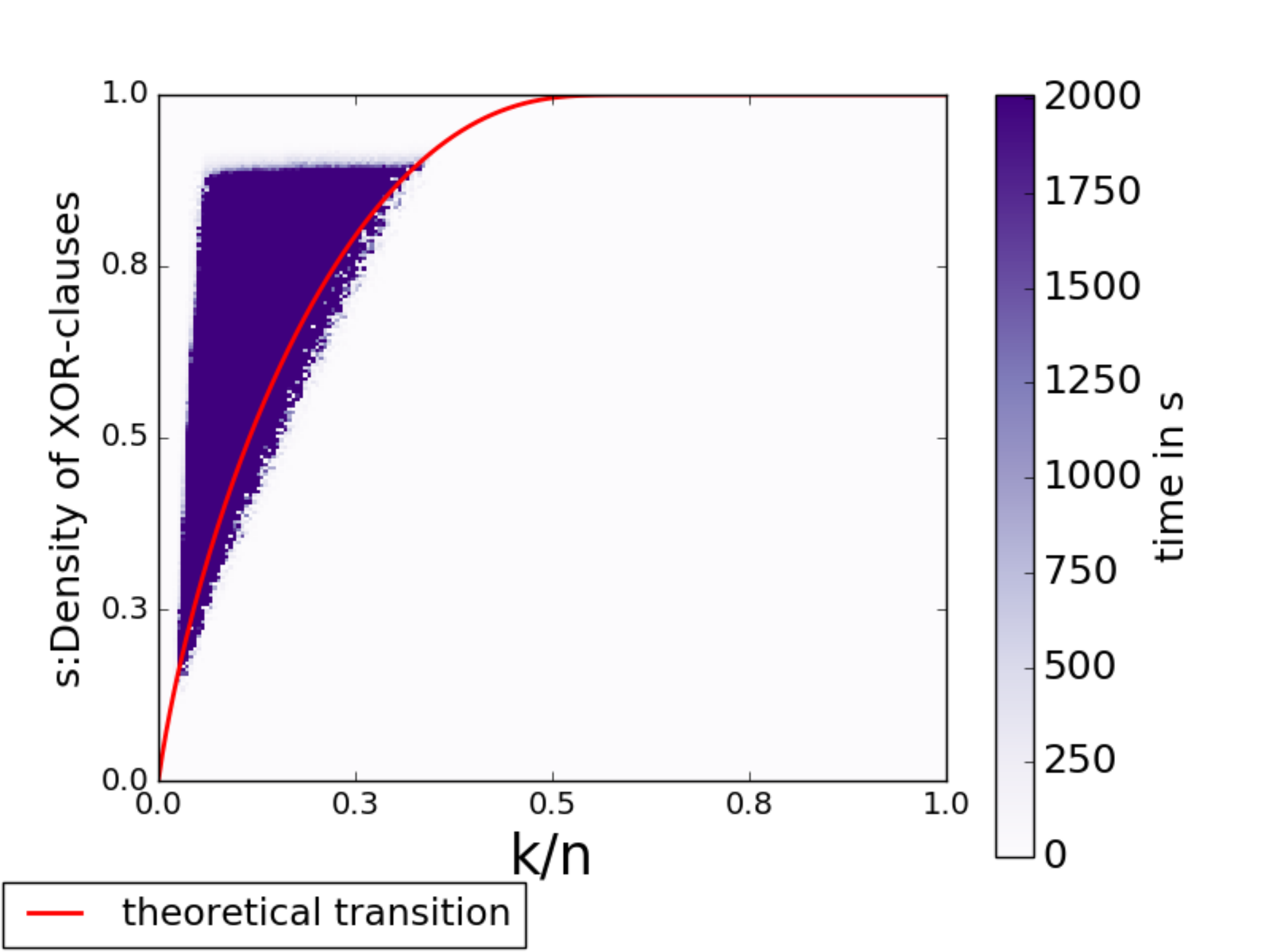}\par
		\captionof{figure}*{\adder}
		\includegraphics[width=1\linewidth]{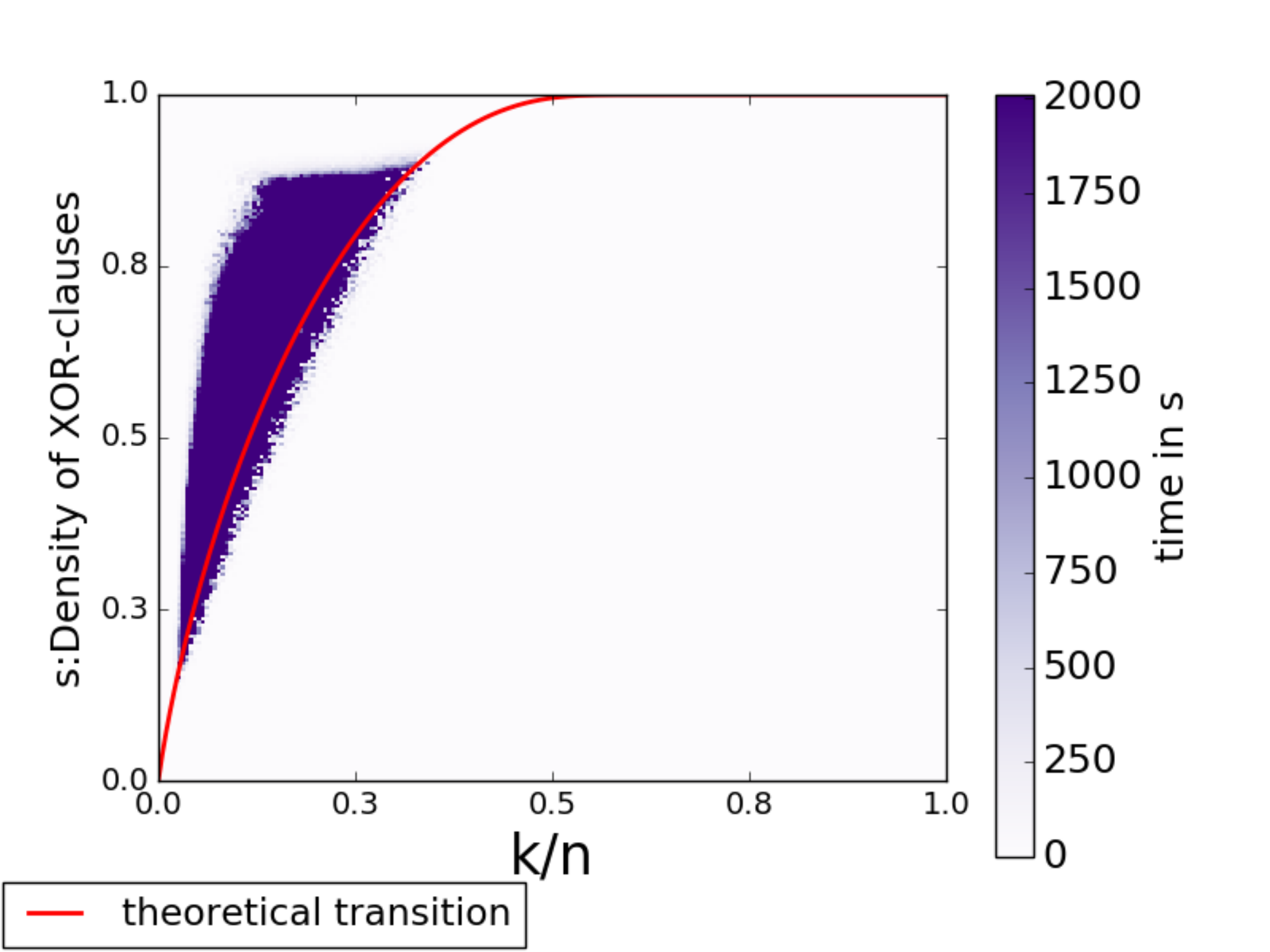}\par 
		\captionof{figure}*{\bdd}
		\includegraphics[width=1\linewidth]{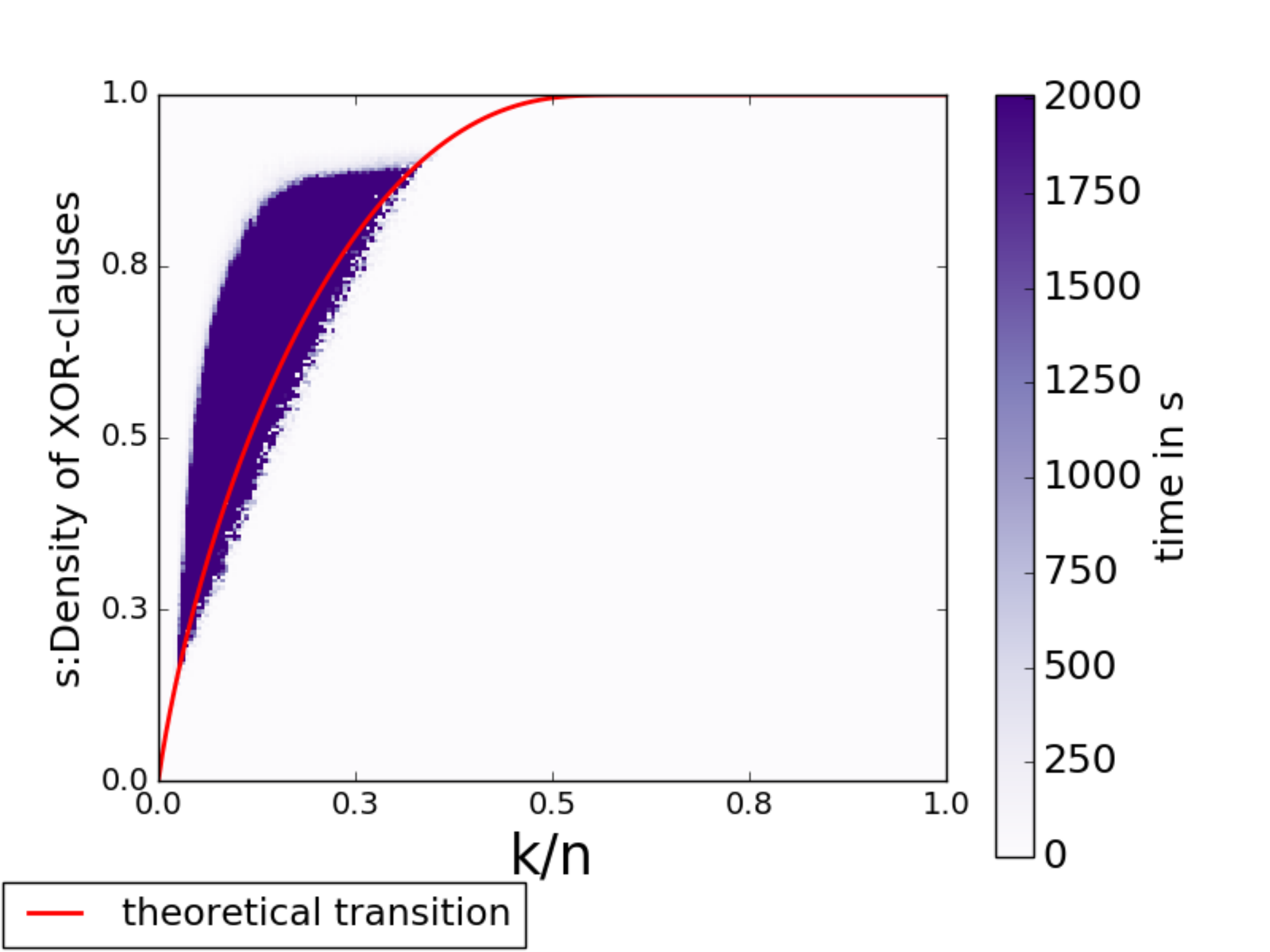}\par
		\captionof{figure}*{\cardnet}
	\end{multicols}
	\caption{Three plots for $n=200$. Each plot shows the runtime behavior for a different encoding in the following order: (a) \adder, (b) \bdd, and (c) \cardnet. The purple region is where the blowup in runtime was observed. The red line indicates the phase transition.\label{fig:enccomp}}
\end{figure*}

\begin{figure}
	\begin{multicols}{2}
		\includegraphics[width=\linewidth,height=\linewidth]{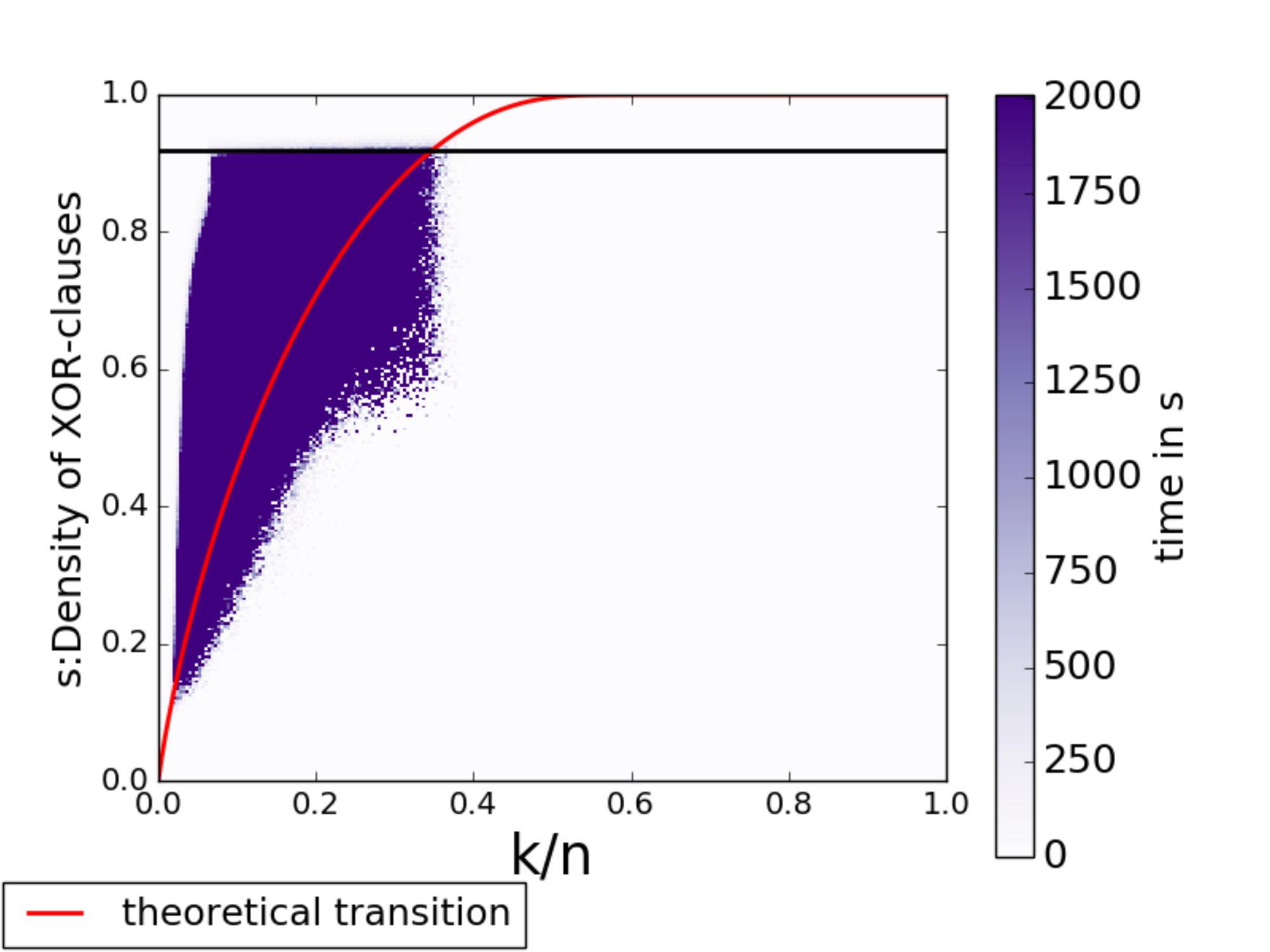}\par
		\captionof{figure}*{(a)}
		\includegraphics[width=\linewidth,height=\linewidth]{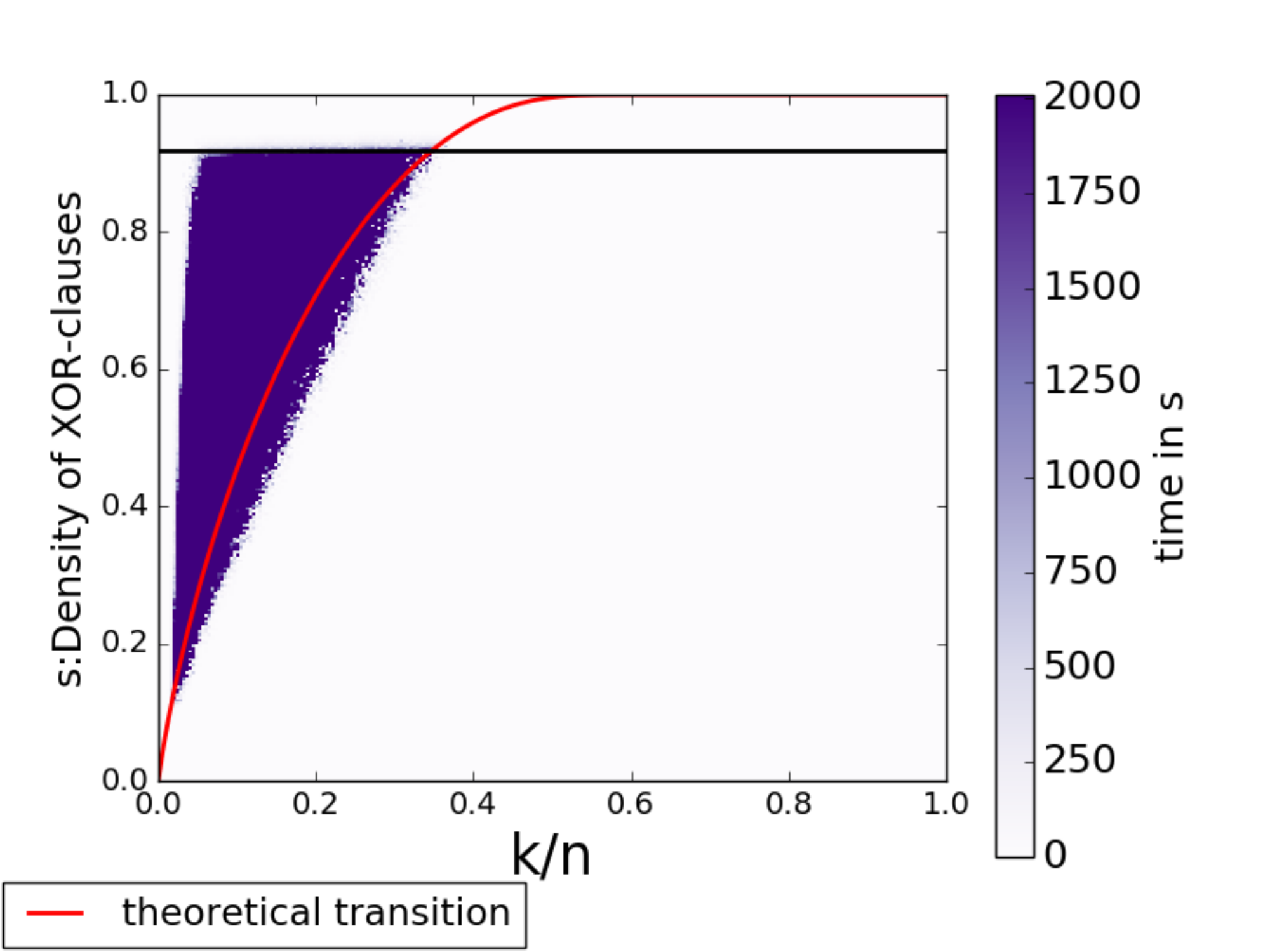}\par
		\captionof{figure}*{(b)}
	\end{multicols}
	\captionof{figure}{$n=250$, \adder. The purple region shows the hard instances for the solver when (a) polarity caching was used and (b)when the solver was made to explore the \false branch first.}	\label{fig:caching} 
\end{figure}

\paragraph{\bf{RQ2.}}We now turn our attention to a study of scaling behavior of runtime performance of SAT solver. Figure~\ref{fig:ncomp} shows the runtime performance for increasing values of $n$, i.e., $n \in \{100,200,300\}$. Similar to Figure~\ref{fig:satisfiablitiy}, the x-axis indicates  $k/n$ while the y-axis indicates the density $s$. Note that most of the instances were solvable in just 2 seconds but the instances within the purple region timed out for timeout of 2000 seconds. Furthermore, we observe that the area of {\em hard} instances increases with $n$. The sudden increase in runtime around the phase transition is reminiscent of similar behavior for random CNF instances but is unlike the behavior observed in case of CNF-XOR formulas~\cite{DMV17}. This behavior necessitates further research for a deeper understanding, and we hope this deeper understanding would be useful in the design of efficient CARD-XOR solvers.  
\paragraph{\bf{RQ3.}}Given the widespread interest in the design of different encodings for CARD constraints, it is natural to ask whether the runtime behavior of $\FQ$ is sensitive to encodings. To this end, Figure~\ref{fig:enccomp} shows the runtime performance of SAT solvers for the above mentioned three different encodings: \adder, \bdd, and \cardnet. Similar to Figure~\ref{fig:satisfiablitiy}, the x-axis indicates  $k/n$ while the y-axis indicates the density $s$. We observe that while the area of the purple region deviates slightly across for different encodings, the qualitative behavior around phase transition regions very similar.

 \paragraph{\bf{RQ4.}} We now turn to the question, what effect do the branching heuristics have on the runtime behavior of SAT solver for $\FQ$. Figure~\ref{fig:caching}(a) shows the behavior when polarity caching was used while Figure~\ref{fig:caching}(b) shows the behavior when the solver always set the {\tt polarity} flag to {\tt false}. The value of $n$ is set to 200 for both the cases. Interestingly, we observe a significant reduction in the area of the purple region by always setting {\tt polarity} flag to {\tt false}, which is surprising given polarity caching has shown to achieve significant gains for SAT solving. A detailed study of different heuristics is beyond the scope of this work and is left for future work.

\section{Conclusion}\label{sec:conclusion}
In this paper, we study 1-CARD-XOR formulas, which are expressed as a conjunction of cardinality constraints and XOR constraints. The CARD-XOR formulas are ubiquitous in several domains of interest such as their close relationship to maximum likelihood decoding and their importance in hashing-based techniques for discrete integration. Our study revealed the empirical existence of phase transition region of satisfiability of random 1-CARD-XOR formulas for which we were able to establish tight theoretical bounds. 

The investigation into runtime behavior led to the surprising discovery of behavior reminiscent of random CNF formulas but significantly different from recent studies on CNF-XOR formulas. Furthermore, we observed that despite significant interest in CP/SAT communities devoted to design of encodings, the qualitative nature of runtime behavior remains consistent across different encodings. Finally, we discover a significant impact of branching heuristics on the runtime behavior. Similar to other CSP problems where the study of phase transition have led to development of algorithmic insights, our study opens future directions into the development of algorithmic techniques for efficient CARD-XOR solvers in practice. 
\section*{Acknowledgements}
This research is supported by the National Research Foundation Singapore under its AI Singapore Programme [R-252-
000-A16-490], the NUS ODPRT Grant [R-252-000-
685-133] and SERB, DST, India through [ECR 2017/001126]. The computational work for this article was performed
on resources of the National Supercomputing Centre, Singapore. \url{https://www.nscc.sg/}.

{\small
	\bibliography{sigproc}
	\bibliographystyle{named}
}
\end{document}